\documentclass{article}
\usepackage{ijcai19}
\pdfoutput=1
\usepackage{times}
\usepackage{soul}
\usepackage{url}

\usepackage[small]{caption}
\usepackage{graphicx}
\usepackage{amsmath}
\usepackage{booktabs}

\usepackage{subfigure}
\usepackage{algorithm}
\usepackage{algorithmicx}
\usepackage{algpseudocode}

\usepackage{fdsymbol}
\usepackage{amsthm}
\usepackage{extarrows}
\usepackage{multirow}

\usepackage{fncylab}

\labelformat{equation}{(#1)} %让所有公式的引用加上小括号

\newtheorem{definition}{Definition}

\newtheorem{theorem}{Theorem}
\newtheorem{corollary}{Corollary}

\usepackage{xcolor}
\title{Causal Discovery with Cascade Nonlinear Additive Noise Models}

\author{
	Ruichu Cai$^1$\and
	Jie Qiao$^1$\and
	Kun Zhang$^2$\and
	Zhenjie Zhang$^3$ \and
	Zhifeng Hao$^4$
	\affiliations
	$^1$School of Computers, Guangdong University of Technology, China\\
	$^2$Department of philosophy, Carnegie Mellon University\\
	$^3$Singapore R\&D, Yitu Technology Ltd.\\
	$^4$School of Mathematics and Big Data, Foshan University, China
	\emails
	cairuichu@gdut.edu.cn,
	qiaojie.chn@gmail.com,
	kunz1@cmu.edu,
	zhenjie.zhang@yitu-inc.com,
	zfhao@gdut.edu.cn
}

\begin{document}
	
	\maketitle
	
	\begin{abstract}
		
		Identification of causal direction between a causal-effect pair from observed data has recently attracted much attention. Various methods based on functional causal models have been proposed to solve this problem, by assuming the causal process satisfies some (structural) constraints and showing that the reverse direction violates such constraints. The nonlinear additive noise model has been demonstrated to be effective for this purpose, but the model class is not transitive--even if each direct causal relation follows this model, indirect causal influences, which result from omitted intermediate causal variables and are frequently encountered in practice, do not necessarily follow the model constraints; as a consequence, the nonlinear additive noise model may fail to correctly discover causal direction. In this work, we propose a cascade nonlinear additive noise model to represent such causal influences--each direct causal relation follows the nonlinear additive noise model but we observe only the initial cause and final effect. We further propose a method to estimate the model, including the unmeasured intermediate variables, from data, under the variational auto-encoder framework. Our theoretical results show that with our model, causal direction is identifiable under suitable technical conditions on the data generation process. Simulation results illustrate the power of the proposed method in identifying indirect causal relations across various settings, and experimental results on real data suggest that the proposed model and method greatly extend the applicability of causal discovery based on functional causal models in nonlinear cases.
		
	\end{abstract}
	
	\section{Introduction}

	Understanding causal relationships is a fundamental problem in various disciplines of science, %Different from the association study, the 
	and causal direction identification is an essential issue in causality studies. It is well known that using randomized experiments to identify causal influences usually encounters unethical or substantial expense issues. Fortunately, inferring causal relations from pure observations, also known as \textit{causal discovery from observational data}, has demonstrated its power in empirical studies and has been a focus in causality research.
	
	%crc：其实这段干脆去掉，加一个related work section？
	Various methods have been proposed to infer the causal direction, by exploring properly constrained forms of functional causal models (FCMs). A functional causal model represents the effect $Y$ as a function of its direct causes $X$ and independent noise, i.e., $Y=f(X;\epsilon), X\Vbar \epsilon$. Without constraints on $f$, then for any two variables one can always express one of them as a function of the other and independent noise \cite{Zhangetal15_TIST}. However, it is interesting to note that with properly constrained FCMs, the causal direction between $X$ and $Y$ is identifiable because the independence condition between the noise and cause holds only for the true causal direction and is violated for the wrong direction. Such FCMs include the Linear, Non-Gaussian, Acyclic Model (LiNGAM)~\cite{shimizu2006linear}, in which $Y = \mathbf{a}^\intercal X +\epsilon$ with linear coefficients $\mathbf{a}$, the nonlinear additive noise model (ANM)~\cite{hoyer2009nonlinear}, in which $Y = f(X)+\epsilon$, and the post-nonlinear (PNL)
	causal model~\cite{zhang2009identifiability}, which also considers possible nonlinear sensor or measurement distortion $f_2$ in the causal process: $Y = f_2(f_1(X)+\epsilon)$. It has been shown that in the generic case, for data generated by the above FCMs, the reverse direction will not admit the same FCM class with independent noise. One can then find causal direction by estimating the FCM followed by testing for independence between the hypothetical cause and estimated noise \cite{hoyer2009nonlinear,zhang2009identifiability}.

	In reality, we can usually record only a subset of all variable which are causally related. If some variable is the direct cause of only one measured variable and is not measured, it is considered as part of the omitted factors, or noise. If a hidden variable is a direct cause of two measured variables, it is a confounder, and causal discovery in the presence of confounders is challenging, although there exist some methods with asymptotic correctness guarantees, such as the FCI algorithm \cite{spirtes2000causation}. In this paper, we are concerned with unmeasured intermediate causal variables. Suppose $X_1 \rightarrow X_2 \rightarrow X_3$, with $X_2$ unmeasured, and that each direct causal influence can be represented by a FCM in a certain class. If the direct causal relations are linear with additive noise, then the causal influence $X_1 \rightarrow X_3$ still follows a linear model with additive noise. However, if each direct causal influence follows the ANM, the causal influence $X_1 \rightarrow X_3$ does not necessarily follow the same model class. Fig.~\ref{fig:illust} gives an illustration of this phenomenon of ``non-transitivity of nonlinear causal model classes," in which $X_2 = 2\tanh(5X_1) + N_2$, and $X_3 = (X_2/2)^3 + N_3$, with $X_1$, $N_2$, and $N_3$ mutually independent and following the uniform distribution between $-0.5$ and $0.5$. As seen from the heterogeneity of the noise in $X_3$ relative to $X_1$, given in Fig.~\ref{fig:illust}(c), the causal influence from $X_1$ to $X_3$ clearly does not admit a nonlinear model with additive noise. Hence, even for the correct causal direction, which is from $X_1$ to $X_3$, the independent noise condition is violated, and existing methods for causal direction determination by checking whether regression residual is independent from the hypothetical cause may fail. The PNL is more general than the additive noise model -- in this example, if $N_3$ is zero, then $X_1 \rightarrow X_3$ will follow this model. However, the PNL model class is also non-transitive. 
	
	\begin{figure}[htbp] \hspace{-3mm}
		
		\centering
		\includegraphics[width=.5\textwidth]{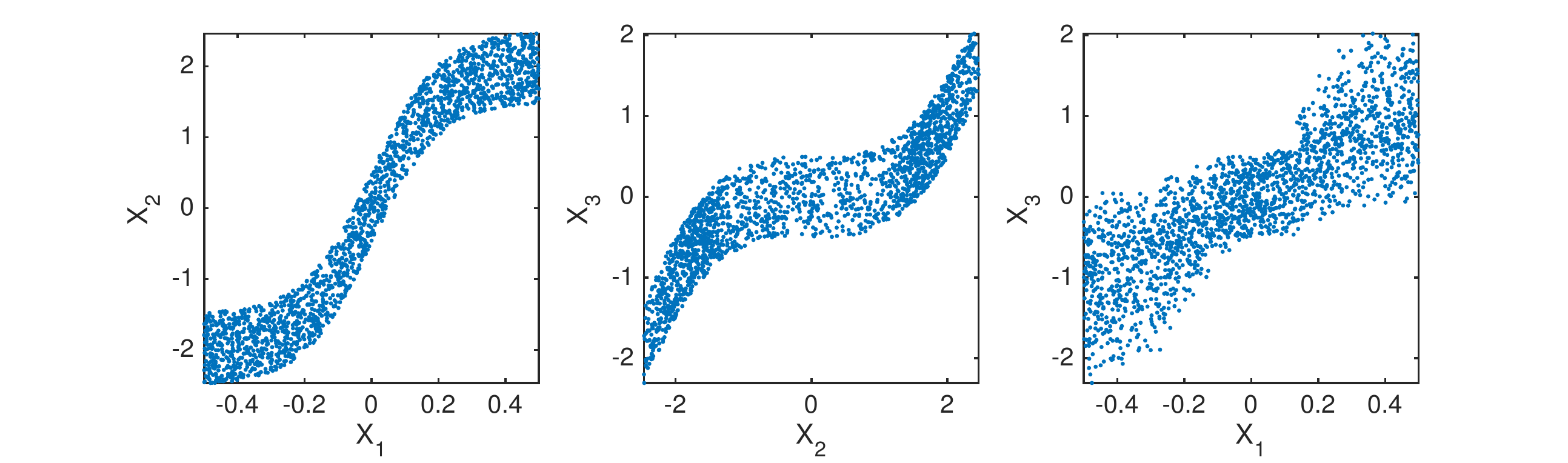}\\
		~~~~(a)~~~~~~~~~~~~~~~~~~~~~~~~(b)~~~~~~~~~~~~~~~~~~~~~~~~(c)
		\caption{Illustration of non-transitivity of nonlinear causal model classes, in which $X_1\rightarrow X_2 \rightarrow X_3$ and each direct causal influence follows a nonlinear model with additive noise. Panels (a), (b), and (c) show the scatter plot of $X_1$ and $X_2 = 2\tanh(5X_1) + N_2$, that of $X_2$ and $X_3 = (X_2/2)^3 + N_3$, and that of $X_1$ and $X_3$, respectively.}
		\label{fig:illust}
		\vspace{-10pt}
	\end{figure}
	
	\begin{figure}[htbp]
		\vspace{-5pt}
		\centering
		\includegraphics[width=0.4\textwidth]{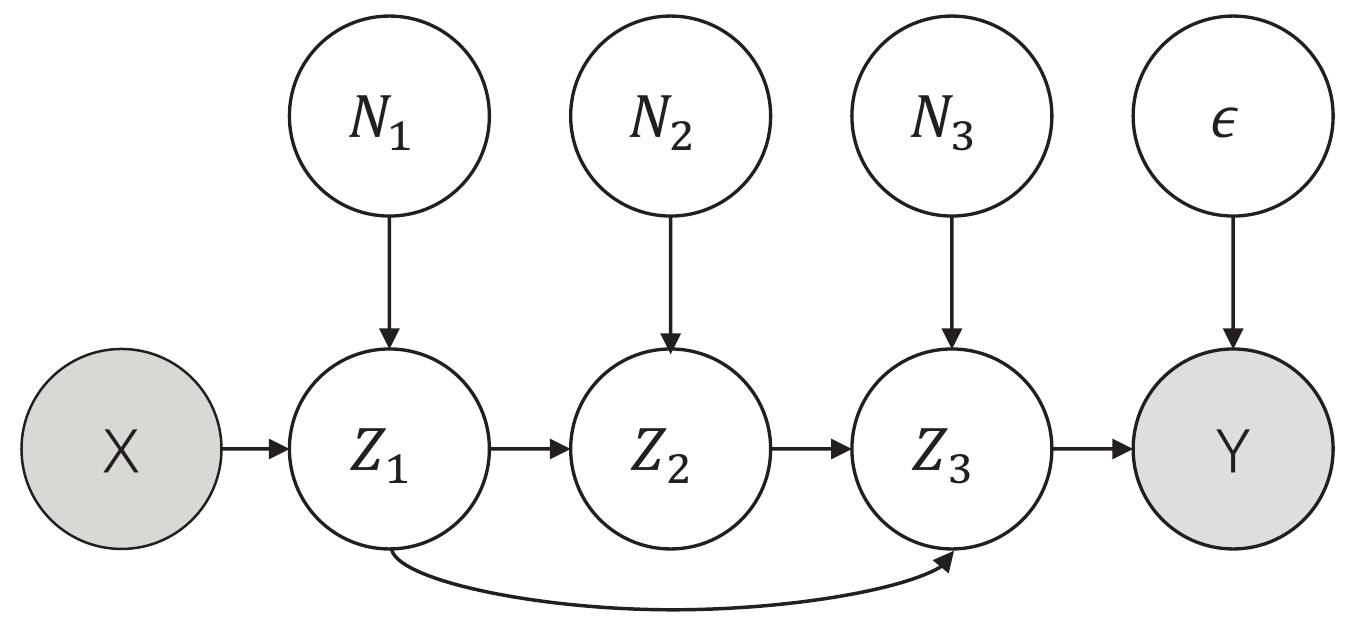}
		\caption{Illustration of the CANM, where the causal chain from $X$ to $Y$ consists of three unmeasured intermediate variables $Z_1,Z_2,Z_3$ with their associated noises $N_1,N_2,N_3$.}
		\label{fig:toy_example}
		\vspace{-5pt}
	\end{figure}

	This paper deals with such indirect, nonlinear causal relations, which seem to be ubiquitous in practice. Finding causal direction for such causal relations has recently been posed as an open problem \cite{spirtes2016causal}. In particular, we aim to find the causal direction between $X$ and $Y$ that are generated according the process given in Fig. \ref{fig:toy_example}, in which there might be a number of unmeasured intermediate causal variables $Z_i$ in between and each direct causal influence, e.g., the influence from $Z_1$ and $Z_2$ on $Z_3$, follows the ANM. We name the causal model from $X$ and $Y$ given in Fig. \ref{fig:toy_example} a Cascade Additive Noise Model (CANM). We note that the considered problem is different from causal discovery in the presence of confounders, for which there have been a number of studies, including the FCI \cite{spirtes2000causation}, RFCI \cite{colombo2012learning}, M3B \cite{yu2018mining} algorithms, and methods relying on stronger assumptions \cite{janzing2009identifying,zhang2010invariant}. \cite{kocaoglu2018entropic} propose an algorithm to search for the latent variable along the path $X$ and $Y$ but they only consider discrete random variables.

	%In conclusion, inferring the causal direction in the cascade additive noise model is still a challenging problem. 
	To the best of our knowledge, this is a first study as to finding causal direction between indirectly and nonlinearly related variables. %By assuming the causal relation following the additive noise model, we successfully extended the nonlinear additive noise model to the indirect cases with unmeasured intermediate variables. 
	The considered causal model can be seen as a cascade of processes, each of which follows the ANM, and the intermediate variables are unmeasured. Intuitively, the independence between the noise and cause is still helpful in finding causal direction--the wrong direction will not follow the independence noise condition in the generic case, allowing us to correctly identify causal direction. This will be supported by our theoretical studies and empirical results in subsequent sections. 
	
	\section{Cascade Additive Noise Model}\label{sec: model}

	Without loss of generality, let $X$ be the cause of effect $Y$ ($X \to Y$), with unmeasured intermediate variables $Z_i$ between them, as shown in Fig.~\ref{fig:toy_example}. We further assume there is no confounder in the mechanism and the data generation follows the nonlinear additive noise assumption. Then, such an indirect causal mechanism can be formalized by the CANM in the following definition.
	
	\begin{definition}
		A CANM for cause $X$ and effect $Y$ is that there exists a sequence of unmeasured intermediate variables between $X$ and $Y$ such that no variable in the latter is the cause of the former one:
		\begin{equation}
		\begin{cases}
		Z_{1} =f_{1}( X) +N_{1}, %&X\Vbar N_1  
		\\
		Z_{t} =f_{t}( \mathbf{Z}_{pa( t)}) +N_{t}, %& \mathbf{Z}_{pa( t)}\Vbar N_t , t=2,3,...,T 
		\\
		Y=f_{T+1}( \mathbf{Z}_{pa( y)}) +\epsilon, %  & \mathbf{Z}_{pa( y)}\Vbar \epsilon ,
		\end{cases}
		\end{equation}
		where $X$, $N_i$, and $\epsilon$ are mutually independent, $T$ denotes depth of the chain, and $\mathbf{Z}_{pa(t)},\mathbf{Z}_{pa(y)}$ denote parents of the $Z_t$ and $y$, respectively. To ensure the cascade structure, 
		the causal relations among $Z_i$ are recursive. %Similarly, $Z_T\subset \mathbf{Z}_{pa(y)}$. 
		Let $\textbf{f}=\{f_1,f_2,...,f_T\}$ and $\textbf{N}=\{N_1,N_2,...,N_T\}$ denote a set of nonlinear functions and the corresponding additive noises at each depth in the chain, respectively. Naturally, here the direct cause and the noises are independent from each other.
	\end{definition}
	
	We are given a set of data $\mathcal{D}=\{x^{(i)} ,y^{(i)}\}_{i=1}^m$. Let $\mathbf{\theta}$ be the parameters of the causal mechanism. Combing all the independence relations of CANM, we can derive its marginal log-likelihood as follows:
	\begin{small}
		\begin{equation}
		\begin{aligned}
		& \log\prod ^{m}_{i=1}\int p_{\mathbf{\theta }} (x^{(i)} ,y^{(i)} ,\mathbf{z} )d\mathbf{z}\\
		= &\log\prod ^{m}_{i=1}\int p_{\mathbf{\theta }} (x^{(i)} )p_{\mathbf{\theta }}(y^{(i)} |\mathbf{z}_{pa(y)} )\prod ^{T}_{t=2} p_{\mathbf{\theta }} (z_{t} |\mathbf{z}_{pa(t)} )p_{\mathbf{\theta }} (z_{1} |x^{(i)} )d\mathbf{z}\\
		= &\log\prod ^{m}_{i=1}\int p (x^{(i)} )p_{\mathbf{\theta }} (\epsilon ^{(i)} =y^{(i)} -f(x^{(i)} ,\mathbf{n} ))\prod ^{T}_{t=1} p_{\mathbf{\theta }} (n_{t} )d\mathbf{n}\\
		= &\log\prod ^{m}_{i=1}\int p_{\mathbf{\theta }} (x^{(i)} ,\epsilon ^{(i)} ,\mathbf{n} )d\mathbf{n}.
		\end{aligned}
		\label{eq:joint likelihood}
		\end{equation}
	\end{small}
	Eq. \ref{eq:joint likelihood} first decomposes the joint likelihood based on the Markov condition \cite{spirtes2000causation}, then applies the independence property between the cause and the noise in the second equality, i.e., $p(Z_{t} |\mathbf{Z}_{pa(t)})=p(N_t=Z_{t}-f_t(\mathbf{Z}_{pa(t)}) |\mathbf{Z}_{pa(t)})\xlongequal {\mathbf{Z}_{pa(t)} \Vbar N_t}p(N_t=Z_{t}-f_t(\mathbf{Z}_{pa(t)}))$. At the same time, we replace $d\mathbf{z}$ with $d\mathbf{n}$ and rewrite function $f_{T+1}(\mathbf{Z}_{pa(y)})$ as $f(X,\mathbf{N})$, because the last unobserved direct cause $\mathbf{Z}_T\subset \mathbf{Z}_{pa(t)}$ contains all the information of the noise $\mathbf{N}$ and cause $X$ relative to $Y$.

	In the above derivation, we used the transformation from $X$ and noises to $Y$. The property of the transformation helps study identifiability and find a practical solution. In light of the independence property of the noises, below we propose a variational approach to approximating the marginal log-likelihood as well as identifying the causal direction.

	\subsection{Variational Solution of CANM}

	The variational solution to estimation of CANM consists of two steps. First, we take advantage of the independence property in CANM to replace the latent variable $\mathbf{Z}$ with $\mathbf{N}$. Second, we find an amortized inference distribution $q_{\phi } (\mathbf{N} |X,Y)$ with respect to the parameter $\phi$ to approximate the intractable posterior $p_{\theta}(\mathbf{N}|X,Y)$ and jointly optimize a variational lower bound of the marginal log-likelihood. Note that, different from the vanilla VAE, $Y$ can be seen as a function of $X$ and $N$ and, as a result, $N$ is a function of X and Y and we need to recover $N$ from both $X$ and $Y$.
	According to Eq. \ref{eq:joint likelihood}, the (log) marginal likelihood, as the sum over of the marginal likelihoods over individual data points:
	\begin{small}
		\begin{equation}
		\label{eq:variational lower bound}
		\begin{aligned}
		& \log\prod ^{m}_{i=1}\int p_{\mathbf{\theta }} (x^{(i)} ,\epsilon ^{(i)} ,\mathbf{n} )d\mathbf{n}\\
		= & \sum ^{m}_{i=1} 
		\underbrace{E_{\mathbf{n} \sim q_{\phi } (\mathbf{n} |x^{(i)} ,y^{(i)} )} \Big[\log \frac{p_{\theta }( x^{(i)} ,\epsilon ^{(i)} ,\mathbf{n})}{q_{\phi }(\mathbf{n} |x^{(i)} ,y^{(i)})}\Big]}_{:=\mathcal{L}\left( \theta ,\phi ;x^{(i)} \!,y^{(i)}\right)} + \\ 
		&~~~~~~~~~~~~~~~~~~KL(q_{\phi } (\mathbf{n} |x^{(i)}\! ,y^{(i)} )\| p_{\theta } (\mathbf{n} |x^{(i)}\! ,y^{(i)} ))\\
		\geqslant  & \sum ^{m}_{i=1}\mathcal{L}\left( \theta ,\phi ;x^{(i)}\! ,y^{(i)}\right),
		\end{aligned}
		\end{equation}
	\end{small}
	where $\mathcal{L}\left( \theta ,\phi ;x^{(i)} ,y^{(i)}\right)$ be the lower bound at data point $(x^{(i)} ,y^{(i)})$, resulting from 
	%%Here $\mathcal{L}$ is the variational lower bound of the marginal likelihood, 
	approximating an intractable posterior $p_{\theta } (\mathbf{n} |x^{(i)} ,y^{(i)} )$ by $q_{\phi } (\mathbf{n} |x^{(i)} ,y^{(i)} )$. Under the framework of CANM, the lower bound of the total marginal likelihood can be further estimated as follows:   
	\begin{equation}
	\label{eq:objective}
	\begin{aligned}
	& \sum ^{m}_{i=1}\mathcal{L}\left( \theta ,\phi ;x^{(i)} ,y^{(i)}\right)\\
	& \begin{aligned}
	=\sum ^{m}_{i=1} E_{\mathbf{n} \sim q_{\phi } (\mathbf{n} |x^{(i)} ,y^{(i)} )} & \left[ -\log q_{\phi }\left(\mathbf{n} |x^{(i)} ,y^{(i)}\right)\right. \\
	& \left. +\log p_{\theta }\left( x^{(i)} ,\epsilon ^{(i)} ,\mathbf{n}\right)\right]
	\end{aligned}\\
	& \begin{aligned}
	= & \sum ^{m}_{i=1}\log p\left( x^{(i)}\right) -KL(q_{\phi } (\mathbf{n} |x^{(i)} ,y^{(i)} )\| p_{\mathbf{\theta }}(\mathbf{n} ))\\
	& +E_{\mathbf{n} \sim q_{\phi } (\mathbf{n} |x^{(i)} ,y^{(i)} )}\left[\log p\left(\epsilon ^{(i)}= y^{(i)} -f\left( x^{(i)} ,\mathbf{n} ;\theta \right)\right)\right].
	\end{aligned}
	\end{aligned}
	\end{equation}
	%%where $\epsilon$ denotes the noise of the latest layer. 
	The details of derivation can be found in Supplementary \ref{sup:elbo}. As shown in Eq. \ref{eq:variational lower bound}, the lower bound $\mathcal{L}$ is tight at $KL(q_{\phi } (\mathbf{n} |x^{(i)},y^{(i)} )\| p_{\mathbf{\theta }} (\mathbf{n} |x^{(i)},y^{(i)} ))=0$. That is, when $q_{\phi } (\mathbf{n} |x^{(i)},y^{(i)} )= p_{\mathbf{\theta }} (\mathbf{n} |x^{(i)},y^{(i)} )$, the lower bound is equal to the marginal log-likelihood. %%Thus, the marginal likelihood can be approximated by maximizing 
	Below we will maximize the variational lower bound.
	
	Here, we assume %%all the prior 
	the distributions of noise $\mathbf{N}$ can be factorized as $p_{\mathbf{\theta }}(\mathbf{N})=\prod_{t=1}^{T}p_{\mathbf{\theta }}(N_t)$. Note that if $\mathbf{N}$ is an empty set, the above lower bound is equivalent to the log-likelihood of the standard additive noise model. 
	
	\subsection{Variational Auto-encoder}

	\begin{figure}[htbp]
		\vspace{-10pt}
		\centering
		\includegraphics[width=0.4\textwidth]{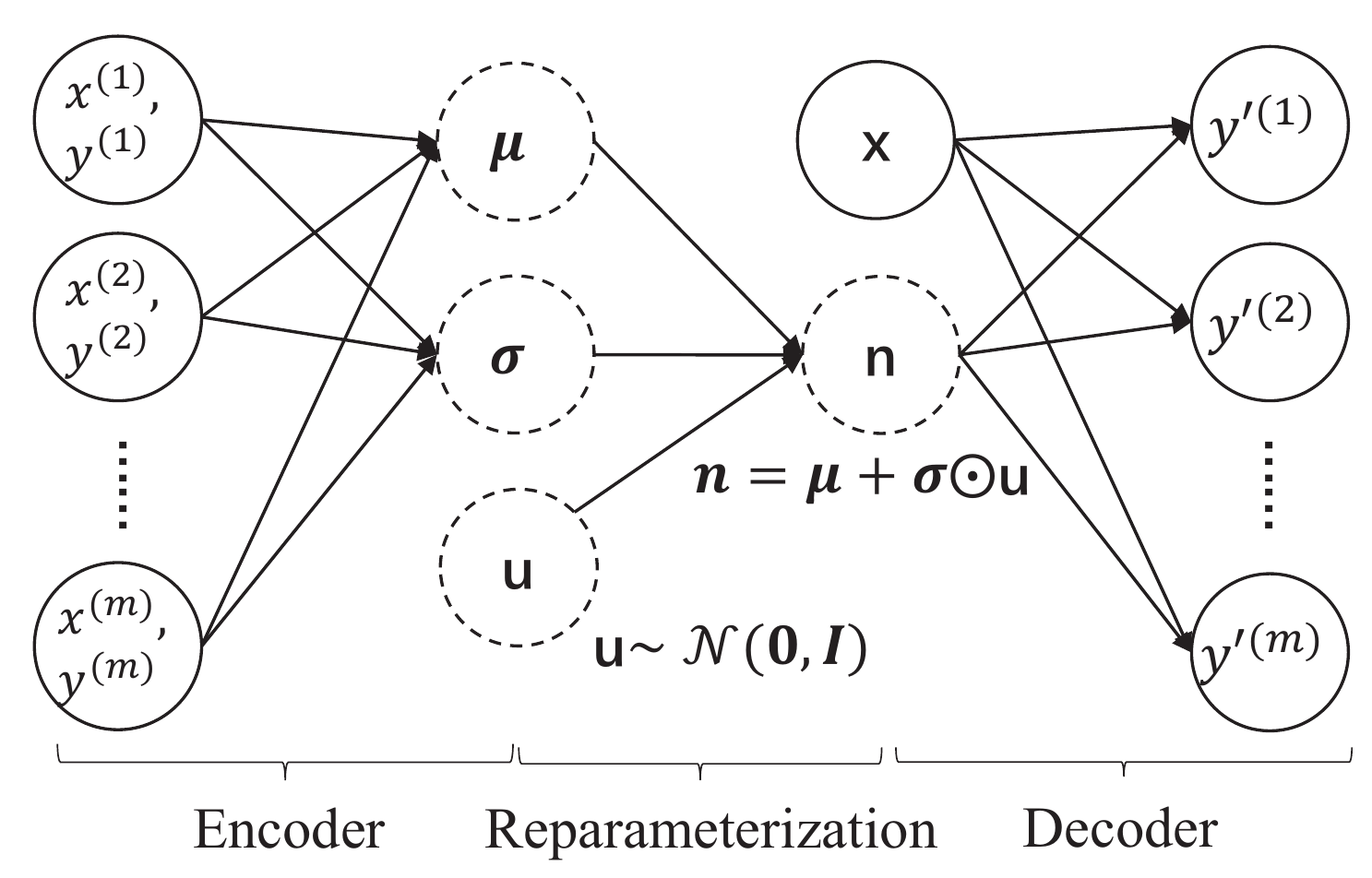}
		\caption{Toy Example for CANM Variational Auto-encoder.}
		\label{fig:vae}
	\end{figure}
	
	The design of the variational auto-encoder (VAE) generally follows the typical configuration in \cite{Kingma2014}. We denote $q_{\phi } (\mathbf{n} |x^{(i)},y^{(i)})$ as \textit{encoder} and $p_{\theta}(y^{(i)}|\mathbf{n},x^{(i)})$ as \textit{decoder}, using a multilayer perceptron (MLP) as an universal approximator for this two functions. 
	
	In the encoder phase, the noises of CANM are inferred by an encoder network with a reparameterization trick. That is, reparameterize the random variable $\mathbf{n} \sim q_{\phi } (\mathbf{n} |x,y)$ with a differentiable transformation ${ h_{\phi } (x,y,u)}$ such that $\mathbf{n} \sim h_{\phi } (x,y,u)$ with $u\sim p(u)$. Then the expectation in the lower bound $E_{\mathbf{n} \sim q_{\phi } (\mathbf{n} |x,y)}\left[ p\left( \epsilon ^{(i)} =y^{(i)} -f\left( x^{(i)} ,\mathbf{n} ;\theta \right)\right)\right]$ can be estimated using Monte Carlo with the reparameterization trick over $L$ samples. 
	
	In the decoder phase, we estimate $\epsilon ^{(i)}$ by calculating the difference between the sample $y^{(i)}$ and the reconstruction from decoder $f\left( x^{(i)} ,h_{\phi } (x^{(i)} ,y^{(i)} ,{\displaystyle u^{(l)}} );\theta \right)$, where $u^{(l)}\sim p(u)$. Finally, through the alternating processing on the encoder and decoder phases, we can optimize the lower bound until it converges.
	
	Fig. \ref{fig:vae} shows a toy example of the structure of the CANM variational auto-encoder with $q_{\phi } (\mathbf{n} |x^{(i)} ,y^{(i)} )=\mathcal{N} (\mathbf{n} ;\mathbf{\mu} _{\phi }\left( x^{(i)} ,y^{(i)}\right) ,\mathbf{\sigma} _{\phi }\left( x^{(i)} ,y^{(i)}\right) \mathbf{I} )$, where $\mathbf{\mu} _{\phi }$ and $\mathbf{\sigma} _{\phi }$ are deterministic function with parameter $\phi$. In the encoder phase, we encode the samples into the noises using a reparameterization trick $\mathbf{n}^{(l)} =\mu _{\phi }\left( x^{(i)} ,y^{(i)}\right) +\mathbf{\sigma} _{\phi }\left( x^{(i)} ,y^{(i)}\right) u^{(l)}$ where $u^{(l)} \sim \mathcal{N}( 0,1)$. In the decoder phase, the sample $y^{(i)}$ is reconstructed by the decoder ${ y\prime ^{(i)} =f\left( x^{(i)} ,\mathbf{n}^{(l)} ;\theta \right)}$.

	\subsection{Practical Algorithm}
	Finally, we propose a general principle that makes use of the VAE to estimate the marginal log-likelihood as well as identify the causal direction.
	\begin{algorithm}[htbp]
		\caption{Inferring causal direction with CANM}
		\label{alg:inference}
		\begin{algorithmic}[1]
			\Require Data samples $\{(x^{(i)},y^{(i)})\}_{i=1}^{m}$.%, the prior of 
			\Ensure The causal direction.
			\State Split the data into training and test sets;
			\State Choose the best number of latent variables by optimizing the variational lower bound (Eq. \ref{eq:variational lower bound}) on the training set and evaluating the performance on the test set; 
			\State  Optimize the lower bound in both directions with the best number of latent variables on the full dataset, obtaining $\mathcal{L}_{X\to Y}$ and $\mathcal{L}_{Y\to X}$ (see Eq. \ref{eq:objective}), respectively. 
			\If {$\mathcal{L}_{X\to Y} > \mathcal{L}_{Y\to X} + \delta$, where $\delta$ is a pre-asigned small positive number,}, 
			\State Infer $X \to Y$
			\ElsIf {$\mathcal{L}_{X\to Y} < \mathcal{L}_{Y\to X} - \delta,$} 
			\State Infer $Y \to X$
			\Else
			\State Non-identifiable
			\EndIf
		\end{algorithmic}
	\end{algorithm}
	
	Algorithm \ref{alg:inference} consists of two phases; the first is model selection, selecting the best number of latent noises, and the second is to identify the causal direction. In  phase 1, by splitting the data into training and testing sets, the best number of noises is selected based on the performance on the test set (Line 1-2). In phase 2, we use the number of the latent noises determined in phase 1 to optimize the variational lower bound on the full dataset and then identify causal direction according to the likelihood for both directions (Line 3-10).
	
	\section{Identifiability}
	In this section, we investigate whether there exist any CANMs whose generated data also admit a CANM in the reverse (anti-causal) direction. In the following theorem, we propose a way to derive the noise distribution for the reverse direction $p(\hat{\epsilon})$ by making use of the theory of Fourier transform \cite{bracewell1986fourier}. The causal direction is unidentifiable according to the CANM if $\hat{\epsilon}$ is independent from $Y$ and $\mathbf{\hat{N}}$ (i.e., the marginal likelihoods for both directions are equal). 
	
	\begin{theorem}
		\label{thm:unidentifiable}
		Let $X\to Y$ follow the cascade additive noise model, while there exists a backward model following the same form, i.e.
		\begin{equation}
		\label{eq:unidentifiable}
		\begin{aligned}[c]
		Y=f( X,\mathbf{N}) +\epsilon,\\
		X=g( Y,\mathbf{\hat{N}}) +\hat{\epsilon}, 
		\end{aligned}
		\qquad
		\begin{aligned}[c]
		X,\mathbf{N}, \textrm{ and }\epsilon \textrm{~are independent}, \\
		Y,\mathbf{\hat{N}}, \textrm{ and }\hat{\epsilon } \textrm{~are independent},
		\end{aligned}
		\end{equation}
		then the noise distribution of the reverse direction $ p_{\hat{\epsilon}}$ must be
		\begin{equation}
		\label{eq:thm1}
		p_{\hat{\epsilon }}\left(\hat{\epsilon }\right)\! =\!\! \int \! e^{2\pi i\hat{\epsilon } \cdot \nu }\frac{\int \!\! \int p( x) p(\mathbf{n}) p_{\epsilon }( y-f( x,\mathbf{n})) e^{-2\pi ix\cdot \nu } d\mathbf{n} dx}{p( y)\int p\left(\hat{\mathbf{n}}\right) e^{-2\pi ig\left( y,\hat{\mathbf{n}}\right) \cdot \nu } d\hat{\mathbf{n}}} d\nu ,
		\end{equation}
		where $f, g$ denote the function implied by the cascade process. 
	\end{theorem}
	\begin{proof}
		See Supplementary \ref{sup:thm1} for a proof.
	\end{proof}

	Roughly speaking, regardless of the linear case, Theorem \ref{thm:unidentifiable} implies that the noise distribution in the reverse direction is generally coherent with $y$. To ensure such noise is independent from $Y$, one strict condition must holds, i.e., $\hat{\epsilon }$ should be independent from $Y$ in the sense that 
	$\forall y_1, y_2, 
	\int \! e^{2\pi i\hat{\epsilon } \cdot \nu }\frac{\int \!\! \int p( x) p(\mathbf{n}) p_{\epsilon }( y_1-f( x,\mathbf{n})) e^{-2\pi ix\cdot \nu } d\mathbf{n} dx}{p( y_1)\int p\left(\hat{\mathbf{n}}\right) e^{-2\pi ig\left( y_1,\hat{\mathbf{n}}\right) \cdot \nu } d\hat{\mathbf{n}}} d\nu 
	=
	\int \! e^{2\pi i\hat{\epsilon } \cdot \nu }\frac{\int \!\! \int p( x) p(\mathbf{n}) p_{\epsilon }( y_2-f( x,\mathbf{n})) e^{-2\pi ix\cdot \nu } d\mathbf{n} dx}{p( y_2)\int p\left(\hat{\mathbf{n}}\right) e^{-2\pi ig\left( y_2,\hat{\mathbf{n}}\right) \cdot \nu } d\hat{\mathbf{n}}} d\nu 
	$. %It should be noted that the general condition for the independence relationship $\hat{\epsilon} \Vbar Y$ holds in CANM is still an open problem.
	However, in general, it seems that such a condition holds only in restrictive cases. Therefore, in most cases, after the latent noise is recovered, we can identify the causal direction by using the independence property for $(X,\mathbf{N}, \epsilon)$. 
	
	To further illustrate the implication of Theorem \ref{thm:unidentifiable}, we provide two special cases in the following corollaries. In Corollary \ref{corollary:linear}, we show that CANM is unidentifiable if the generation process is linear Gaussian. In Corollary \ref{corollary:anm}, we show the connection with ANM when there is no unmeasured intermediate variables, and shows a generic choices of $f$, $p_X(x)$, and $p_\epsilon(\epsilon)$ for the identification of the model. Those two special cases are consistent with the previous results.
	\begin{corollary}\label{corollary:linear}
		Assume that CANM is linear Gaussian, i.e., 
		\begin{equation*}
		Y=aX+bN+\epsilon ,
		\end{equation*}
		where $X,N,\epsilon \sim \mathcal{N}(0,1)$, then their exist a backward CANM
		\begin{equation*}
		X=\frac{a}{a^{2} +b^{2} +1}Y+\frac{a}{\sqrt{a^{2} +b^{2} +1}}\hat{N}+\hat{\epsilon},
		\end{equation*} 
		where $\hat{N}, \hat{\epsilon} \sim \mathcal{N}(0,1)$ and $\hat{\epsilon }$ is independent of $Y$ and $\hat{N}$.
		
	\end{corollary}
	\begin{proof}
		See Supplementary \ref{sup:cor1} for a proof.
	\end{proof}
	
	\begin{corollary}
		\label{corollary:anm}
		Suppose that there is no unmeasured intermediate noises in CANM, if the solution of Eq. \ref{eq:thm1} exists, then the triple $\displaystyle ( f,p_{X} ,p_{\epsilon})$ must satisfy the differential equation from ANM \cite[Theorem~1]{hoyer2009nonlinear} for all $\displaystyle x,y$ with $\displaystyle \nu ''( y-f( x)) f'( x) \neq 0$:
		\begin{equation}
		\xi ^{\prime \prime \prime } =\xi ^{\prime \prime }\left( -\frac{\nu ^{\prime \prime \prime } f^{\prime }}{\nu ^{\prime \prime }} +\frac{f^{\prime \prime }}{f^{\prime }}\right) -2\nu ^{\prime \prime } f^{\prime \prime } f^{\prime } +\nu ^{\prime } f^{\prime \prime \prime } +\frac{\nu ^{\prime } \nu ^{\prime \prime \prime } f^{\prime \prime } f^{\prime }}{\nu ^{\prime \prime }} -\frac{\nu ^{\prime }\left( f^{\prime \prime }\right)^{2}}{f^{\prime }} ,
		\end{equation}
		where $\displaystyle \nu \coloneq \log p_{\epsilon} ,\ \xi \coloneq \log p_{X}$ 
	\end{corollary}
	\begin{proof}
		See Supplementary \ref{sup:cor2} for a proof.
	\end{proof}

	\section{Experiments}\label{sec:exp}

	\subsection{Synthetic Data}
	% 产生数据的方法，变深度，变样本两个实验，然后核密度估计的方式，
	In this section, we design three experiments with known ground truth, with the depth $=\{0,1,2,\textbf{3},4,5\}$, sample size $=\{250,500,1000,2000,\textbf{3000},4000,5000,6000\}$, and with different sample sizes for some fix structures. The default setting is marked in bold. All the experimental results are averaged over 1000 random generated causal pairs generated by the cascade additive noise model. Code for CANM is available online\footnote{\url{https://github.com/DMIRLAB-Group/CANM}}.
	
	To make the synthetic data general enough, in each depth, we randomly generate an additive noise model and then stack it together to obtain the cascade additive noise model. In detail, the cause $(X)$ is sampled from a random Gaussian Mixture model of 3 components $p(x_i|\theta)=\sum_{k=1}^{3} \pi_k \mathcal{N}(x_i|\mu_k,\sigma_k)$ where $\mu_k\sim \mathcal{N}(0,1),\sigma_k \sim Super-Gaussian$. For each layer $x_{t}=f_{t}(x_{t-1})+n_t$ where $n_t\sim \mathcal{N}(0,1)$ and $f_t$ is generated from a cubic spline interpolation using a 6-dimensional grid from $\min (x_{t-1})$ to $\max (x_{t-1})$ as input with respect to 6 random generated points as knots for the interpolation; the generated points are sampled from $\mathcal{N}(0,1)$ and the number of knots is used to control non-linearity of the function. Such generative process follows the instrument given in \cite{prestwich2016causal}.
	
	The following four algorithms are taken as baseline methods: ANM \cite{hoyer2009nonlinear}, CAM \cite{buhlmann2014cam}, IGCI \cite{janzing2012information}, and LiNGAM \cite{shimizu2006linear}. We also improve the implementation for ANM by using the XGBoost \cite{chen2016xgboost} for regression and the Hilbert-Schmidt independence criterion (HSIC) \cite{gretton2008kernel} as the independence test. Therefore, ANM can be evaluated in two ways. First, we compare the HSIC statistic to determine the direction and second, we select the best significance level ($p=0.01$) range from 0.01 to 1 to determine the causal direction. At the same time, the best parameter setting of IGCI is chosen. For the other baseline methods, we use the parameter settings in their original papers. The implementation and the parameter settings of LiNGAM and CAM are based on the CompareCausalNetworks packages in R \cite{CompareCausalNetworks}.
	
	% 分析
	
	\begin{figure}[http]
		
		\centering
		\includegraphics[width=0.4\textwidth]{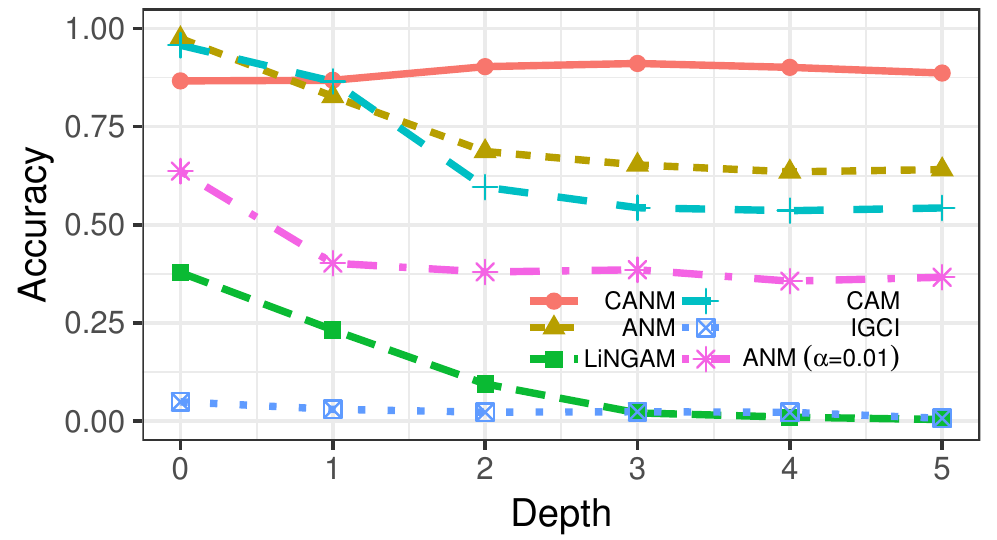}
		
		\caption{Sensitivity to Depth.}
		\label{fig:sensitivity_depth}
		
	\end{figure}
	
	\textbf{Sensitivity to Depth}: Fig. \ref{fig:sensitivity_depth} shows the accuracy with different depths in 3000 samples. Firstly, when the depth is equal to 0 (the original additive noise model), all CANM, ANM, and CAM achieve a high accuracy. Note that CANM still has a similar performance comparing with ANM even though CANM assume that there might exist unmeasured intermediate variables, which demonstrates the robustness of our method. Secondly, as the depth increases, the accuracy of CANM is stable and around 90\% accuracy with a slight decrease, while the performance of the rest methods decreases rapidly as the depth grows. In particular, the ANM with the significance level of 0.01 gives almost random decisions when the cascade structure exists.
	
	\begin{figure}[http]
		\vspace{-10pt}
		\centering
		\includegraphics[width=0.4\textwidth]{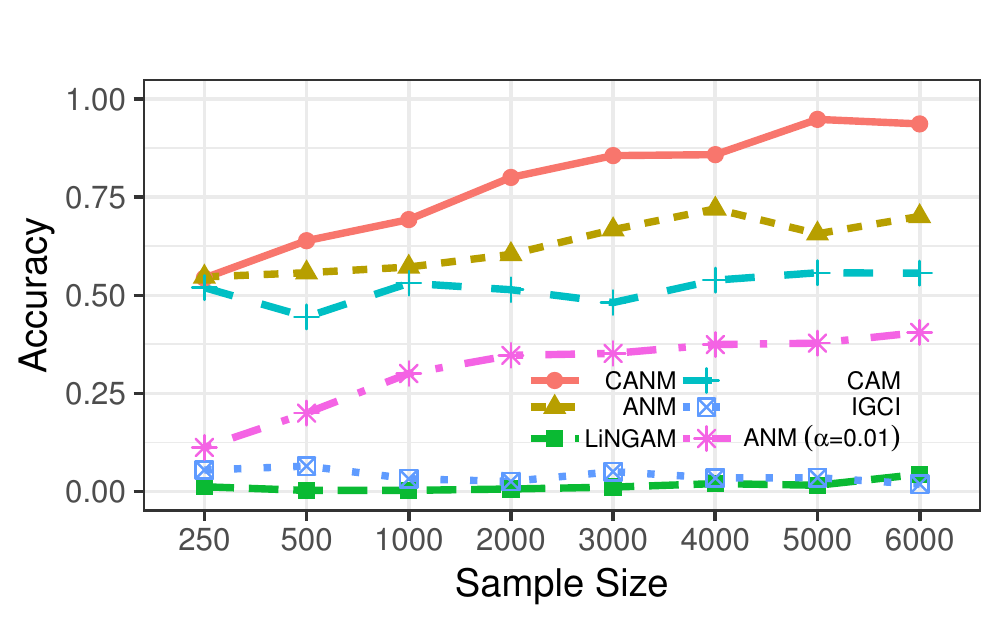}
		\vspace{-8pt}
		\caption{Sensitivity to Sample.}
		\label{fig:sensitivity_sample}
		\vspace{-10pt}
	\end{figure}
	
	\textbf{Sensitivity to Sample Size}: Fig. \ref{fig:sensitivity_sample} shows the accuracy with different sample sizes while the depth is fixed at 3. The result shows that even in the small sample size, CANM still outperforms the other methods. As the sample size increases, the accuracy of CANM grows faster than the other methods. Thus, large samples are beneficial to CANM, because of the variational auto-encoder framework employed in CANM. A similar result also can be observed in ANM and CAM while the other methods are less sensitive to the sample size due to the model restriction. 
	
	\begin{figure}[http]
		\vspace{-5pt}
		\centering
		\includegraphics[width=0.4\textwidth]{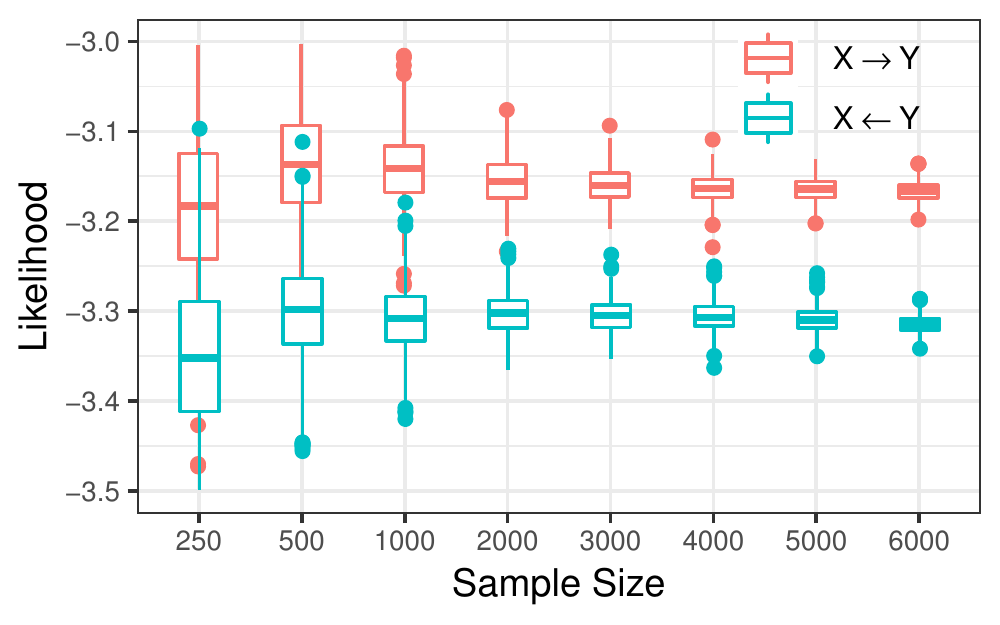}
		\vspace{-5pt}
		\caption{Sensitivity to Sample in a Fixed Structure.}
		\label{fig:fixF_sample}
		\vspace{-5pt}
	\end{figure}
	
	\textbf{Sensitivity to Sample Size in a Fixed Structure}: Fig. \ref{fig:fixF_sample} shows the accuracy with different numbers of samples while we use a fixed causal mechanism, which was randomly generated with depth=3. When the sample size is small, the variance of the likelihood is large; however, the asymmetry in the causal direction is still clear. As the sample size increases, the variance of the likelihood decreases and the accuracy increases, which implies the effectiveness and robustness of CANM as the sample size grows.
	
	\subsection{Real World Data}
	% 这个数据集是energy industry 的数据，包含了温度，电力load，根据常识我们知道
	\textbf{Electricity consumption}: The electricity consumption dataset \cite{prestwich2016causal} has 9504-hour measurements from the energy industry, containing the $hour~of~data$, outside $temperature$ and the $electricity~load$ on the power station. The causal mechanism among the three variables is $hour~of~day \to temperature $ and $temperature \to electricity~load$. The first pair is generally caused by the heating of sunlight and the second pair is base on the fact that the usage of heating or air condition depends on the temperature. We are interested to know whether we can identify the $hour~of~day~(X)$ is the cause of the $electricity~load~(Y)$ and what intermediate variable will be inferred via CANM. 
	
	\begin{figure}[http]
		
		\centering
		\includegraphics[width=0.4\textwidth]{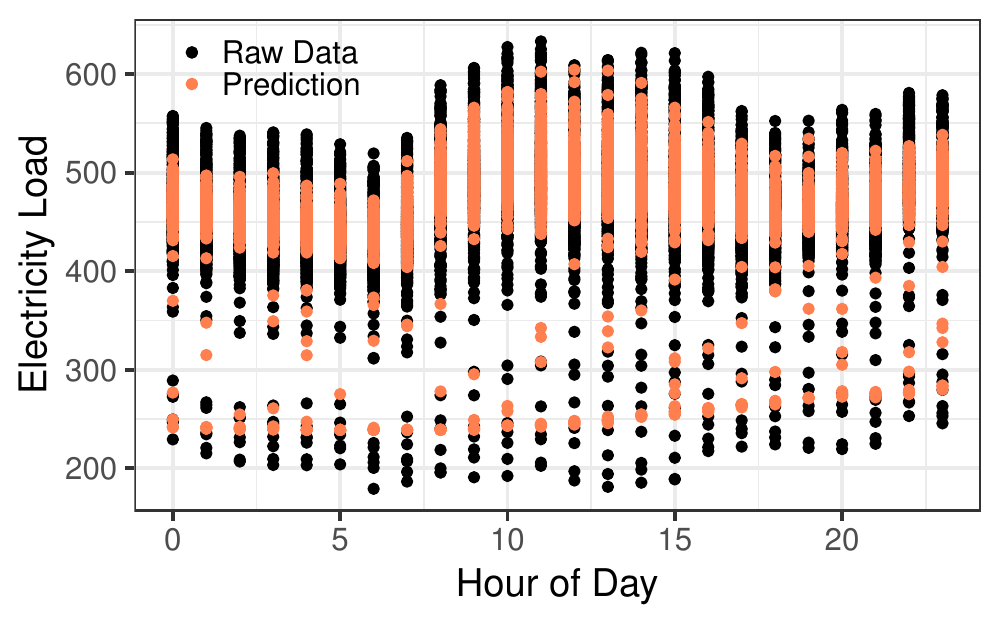}
		\vspace{-5pt}
		\caption{Hour of Day Against Electricity Load.}
		\label{fig:real_pair9496_pred}
	\end{figure}
	
	\begin{figure}[http]
		\vspace{-5pt}
		\centering
		\includegraphics[width=0.4\textwidth]{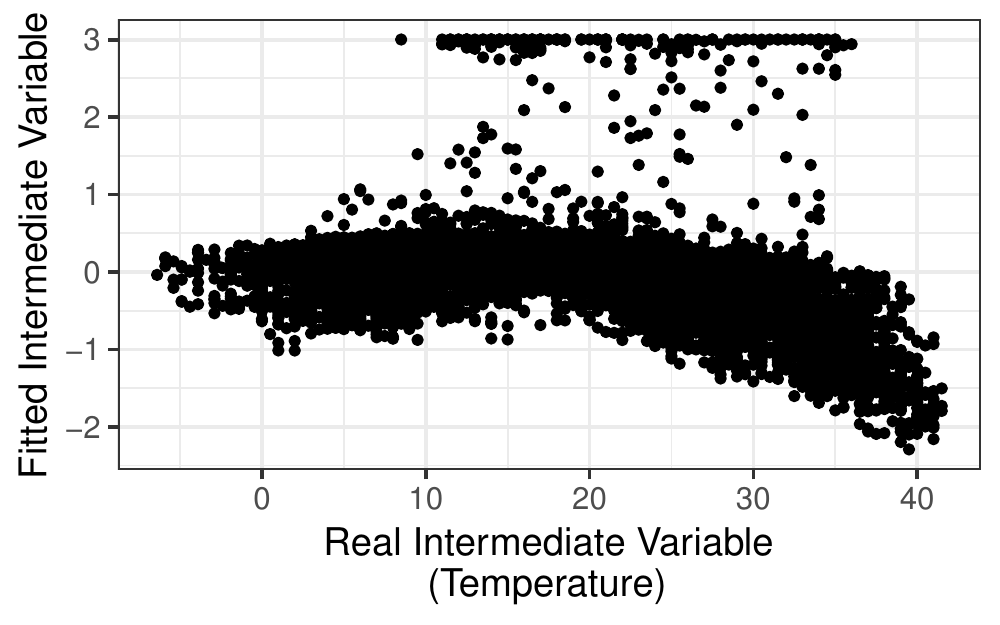}
		\caption{Temperature Against Fitted Intermediate Variable.}
		\label{fig:real_pair9496_z}
		\vspace{-10pt}
	\end{figure}
	
	In general, we successfully identify the correct causal direction with average score $\mathcal{L}_{X\to Y}=-2.62>\mathcal{L}_{Y\to X}=-2.67$ while ANM fails on this pair (the p-value $=0$ on both directions). The prediction of $electricity$ is given in Fig. \ref{fig:real_pair9496_pred}. It is interesting to note that there might exist more than one unmeasured variable, e.g., season, causing a different electricity load at the same hour of day. Such unmeasured variables are successfully captured by CANM as the prediction separating into both upper and lower parts. Furthermore, the intermediate variable inferred by our method has rather high correlation ($\rho=-0.35$) with the temperature as shown in Fig. \ref{fig:real_pair9496_z}, which means that CANM not only recovers the information of the season but also the information of the temperature.
	
	\textbf{Stock Market}: The stock market dataset is collected by T\"ubingen causal effect benchmark (\url{https://webdav.tuebingen.mpg.de/cause-effect/}) as pairs 66-67. It contains the stock return of $Hutchison$, $Cheung~Kong$ and $Sun~Hung~Kai$ with the causal relationship: $Hutchison \rightarrow Cheung~Kong$ and $Cheung~Kong \rightarrow Sun~Hung~Kai$. The reason for the first pair is that Cheung Kong owns about 50\% of Hutchison. For the second pair, Sun Hung Kai Prop., a typical stock in the Hang Seng Property subindex, is believed to depend on the major stock Cheung Kong. Similarly to the previous experiment, we are interested to know whether we can identify the $Hutchison~(X)$ is the cause of the $Sun~Hung~Kai~(Y)$.
	
	\begin{figure}[http]
		\centering
		\includegraphics[width=0.4\textwidth]{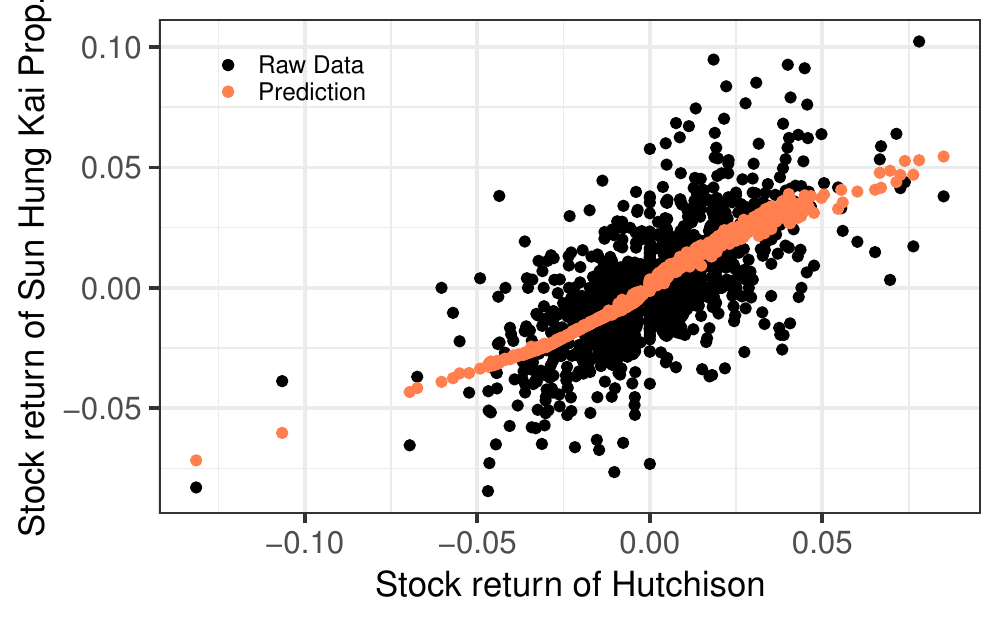}
		\vspace{-5pt}
		\caption{Stock return of Hutchison Against Stock return of Sun Hung Kai Prop.}
		\label{fig:real_pair6667_pred}
		%\vspace{-5pt}
	\end{figure}
	
	\begin{figure}[http]
		\vspace{-5pt}
		\centering
		\includegraphics[width=0.4\textwidth]{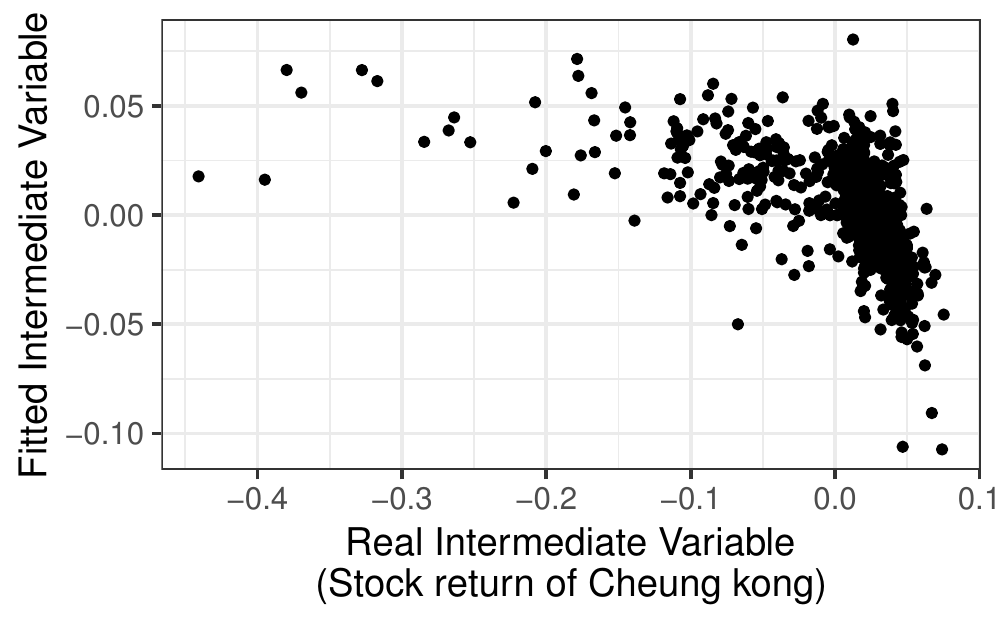}
		\vspace{-5pt}
		\caption{Stock return of Cheung kong Against Fitted Intermediate Variable.}
		\label{fig:real_pair6667_z}
		\vspace{-10pt}
	\end{figure}
	Since these three stocks form a causal chain that $Hutchison \rightarrow Cheung~Kong \rightarrow Sun~Hung~Kai$, using CANM, we successfully identify the indirect causal direction with average score $\mathcal{L}_{X\to Y}=-2.49>\mathcal{L}_{Y\to X}=-2.51$ while ANM fails on this pair (the p-value $=0.006< 0.05$ on the causal direction and p-value $=0.29> 0.05$ on the reverse direction). Fig. \ref{fig:real_pair6667_pred} shows the prediction of the stock return of the $Sun~Hung~Kai$. We also find that the fitted intermediate variable has a high correction ($\rho=-0.54$) with the stock return of $Cheung~Kong$ as shown in Fig. \ref{fig:real_pair6667_z}.
	
	\vspace{-5pt}

	\section{Conclusion}\label{sec:conclu}
	In this paper, we proposed the cascade nonlinear additive noise model, as an extension of the nonlinear additive noise model, to represent indirect causal influences, which result from unmeasured intermediate causal variables. We have demonstrated that, the independence between the noise and cause is still generally helpful to determine causal direction between two variables, as long as the cascade additive noise process holds. We propose to estimate the model as well as the intermediate causal variables with the variational auto-encoder framework, and the produced likelihood indicates the asymmetry between cause and effect. As supported by our theoretical and empirical results, the proposed approach provides an effective method for causal direction determination from data generated by nonlinear, indirect causal relations.
	
	\section*{Acknowledgments}
	%\vspace{-3pt}
	This research was supported in part by NSFC-Guangdong Joint Found (U1501254), Natural Science Foundation of China (61876043), Natural Science Foundation of Guangdong (2014A030306004, 2014A030308008),  Guangdong High-level Personnel of Special Support Program (2015TQ01X140) and Pearl River S\&T Nova Program of Guangzhou (201610010101). KZ would like to acknowledge the support by National Institutes of Health (NIH) under Contract No. NIH-1R01EB022858-01, FAINR01EB022858, NIH-1R01LM012087, NIH-5U54HG008540-02, and FAIN- U54HG008540, by the United States Air Force under Contract No. FA8650-17-C-7715, and by National Science Foundation (NSF) EAGER Grant No. IIS-1829681. The NIH, the U.S. Air Force, and the NSF are not responsible for the views reported here.
	We appreciate the comments from anonymous reviewers, which greatly helped to improve the paper. 
	
	%% The file named.bst is a bibliography style file for BibTeX 0.99c
	\bibliographystyle{named}
	\bibliography{ijcai19_arxiv}
	
	\newpage 
	\onecolumn
	\section*{Supplementary Material}
	
	\setcounter{equation}{0}
	\renewcommand{\theequation}{S.\arabic{equation}}
	\setcounter{figure}{0}
	\renewcommand{\thefigure}{S.\arabic{figure}}
	%\part{}
	%\stepcounter{section}
	\renewcommand{\thesection}{\Alph{section}}
	%\counterwithin*{section}{part}
	\setcounter{section}{0}
	
	\section{The Lower Bound for Cascade Nonlinear Additive Noise Model}\label{sup:elbo}

	\begin{equation*}
	\begin{aligned}
	& \log p(x^{(i)} ,y^{(i)} )\\
	= & \log p(x^{(i)} ,\epsilon ^{(i)} ,\mathbf{n} )-\log p(\mathbf{n} |x^{(i)} ,y^{(i)} )\\
	= & \log\left(\frac{p(x^{(i)} ,\epsilon ^{(i)} ,\mathbf{n} )}{q (\mathbf{n} |x^{(i)} ,y^{(i)} )}\right) -\log\left(\frac{p(\mathbf{n} |x^{(i)} ,y^{(i)} )}{q (\mathbf{n} |x^{(i)} ,y^{(i)} )}\right)\\
	= & \log p(x^{(i)} ,\epsilon ^{(i)} ,\mathbf{n} )-\log q (\mathbf{n} |x^{(i)} ,y^{(i)} )-\log\left(\frac{p(\mathbf{n} |x^{(i)} ,y^{(i)} )}{q (\mathbf{n} |x^{(i)} ,y^{(i)} )}\right)\\
	= & \log p(x^{(i)} ,\epsilon ^{(i)} )+\log p(\mathbf{n} )-\log q (\mathbf{n} |x^{(i)} ,y^{(i)} )-\log\left(\frac{p(\mathbf{n} |x^{(i)} ,y^{(i)} )}{q (\mathbf{n} |x^{(i)} ,y^{(i)} )}\right)\\
	= & \log p\left( \epsilon ^{(i)} =y-f\left(\mathbf{n} ,x^{(i)}\right)\right) +\log p\left( x^{(i)}\right) +\log p(\mathbf{n} )-\log q (\mathbf{n} |x^{(i)} ,y^{(i)} )-\log\left(\frac{p(\mathbf{n} |x^{(i)} ,y^{(i)} )}{q (\mathbf{n} |x^{(i)} ,y^{(i)} )}\right)\\
	= & \log p\left( x^{(i)}\right) +\int q (\mathbf{n} |x^{(i)} ,y^{(i)} )\log p\left( \epsilon ^{(i)} =y^{(i)} -f\left(\mathbf{n} ,x^{(i)}\right)\right) d\mathbf{n}\\
	& +\int q (\mathbf{n} |x^{(i)} ,y^{(i)} )\log\frac{p(\mathbf{n} )}{q (\mathbf{n} |x^{(i)} ,y^{(i)} )} d\mathbf{n} -\int q (\mathbf{n} |x^{(i)} ,y^{(i)} )\log\left(\frac{p(\mathbf{n} |x^{(i)} ,y^{(i)} )}{q (\mathbf{n} |x^{(i)} ,y^{(i)} )}\right) d\mathbf{n}\\
	= & \log p\left( x^{(i)}\right) +E_{\mathbf{n} \sim q (\mathbf{n} |x^{(i)} ,y^{(i)} )}\left[\log p\left( \epsilon ^{(i)} =y^{(i)} -f\left(\mathbf{n} ,x^{(i)}\right)\right)\right] -KL(q (\mathbf{n} |x^{(i)} ,y^{(i)} )\| p(\mathbf{n} ))+KL(q (\mathbf{n} |x^{(i)} ,y^{(i)} )\| p(\mathbf{n} |x^{(i)} ,y^{(i)} ))\\
	\geqslant  & \log p\left( x^{(i)}\right) +E_{\mathbf{n} \sim q (\mathbf{n} |x^{(i)} ,y^{(i)} )}\left[\log p\left( \epsilon ^{(i)} =y^{(i)} -f\left(\mathbf{n} ,x^{(i)}\right)\right)\right] -KL(q (\mathbf{n} |x^{(i)} ,y^{(i)} )\| p(\mathbf{n} ))
	\end{aligned}
	\end{equation*}
	\section{Proof of Theorem 1}\label{sup:thm1}
	\begin{theorem}
		\label{sup:thm:unidentifiable}
		Let $X\to Y$ follow the cascade additive noise model, while there exists a backward model following the same form, i.e.
		\begin{equation}
		\label{sup:eq:unidentifiable}
		\begin{aligned}[c]
		Y=f( X,\mathbf{N}) +\epsilon,\\
		X=g( Y,\mathbf{\hat{N}}) +\hat{\epsilon}, 
		\end{aligned}
		\qquad
		\begin{aligned}[c]
		X,\mathbf{N}, \textrm{ and }\epsilon \textrm{~are independent}, \\
		Y,\mathbf{\hat{N}}, \textrm{ and }\hat{\epsilon } \textrm{~are independent},
		\end{aligned}
		\end{equation}
		then the noise distribution of the reverse direction $ p_{\hat{\epsilon}}$ must be
		\begin{equation}
		\label{sup:eq:thm1}
		p_{\hat{\epsilon }}\left(\hat{\epsilon }\right)\! =\!\! \int \! e^{2\pi i\hat{\epsilon } \cdot \nu }\frac{\int \!\! \int p( x) p(\mathbf{n}) p_{\epsilon }( y-f( x,\mathbf{n})) e^{-2\pi ix\cdot \nu } d\mathbf{n} dx}{p( y)\int p\left(\hat{\mathbf{n}}\right) e^{-2\pi ig\left( y,\hat{\mathbf{n}}\right) \cdot \nu } d\hat{\mathbf{n}}} d\nu ,
		\end{equation}
		where $f, g$ denote the function implied by the cascade process. 
	\end{theorem}
	\newenvironment{proof-sketch}{\noindent{ \textit{Sketch of Proof:}}\hspace*{1em}}{\qed\bigskip}
	\begin{proof-sketch}
		Based on the derivation of the marginal log-likelihood at Eq. \ref{eq:joint likelihood} in Section \ref{sec: model}, if Eq. \ref{sup:eq:unidentifiable} holds, we have
		
		\begin{equation} 
		\begin{aligned}
		p( y|x) =\int p(\mathbf{n}) p_{\epsilon }(\epsilon= y-f( x,\mathbf{n})) d\mathbf{n},\\
		p( x|y) =\int p(\mathbf{\hat{n}}) p_{\hat{\epsilon }}(\hat{\epsilon}= x-g\left( y,\mathbf{\hat{n}}\right)) d\hat{\mathbf{n}}.
		\end{aligned}
		\end{equation}
		Applying Fourier transform to $\displaystyle p( x|y)$, we obtain
		\begin{equation}
		\label{eq:fourier1}
		\begin{aligned}
		\mathcal{F}( \nu ) & =\int p( x|y) e^{-2\pi ix\cdot \nu } dx\\
		& =\int p\left(\mathbf{\hat{n}}\right)\int p_{\hat{\epsilon }}\left( x-g\left( y,\mathbf{\hat{n}}\right)\right) e^{-2\pi ix\cdot \nu } dxd\hat{\mathbf{n}}.
		\end{aligned}
		\end{equation}
		Since $\displaystyle \hat{\epsilon } =x-g\left( y,\hat{\mathbf{n}}\right)$, we have $\displaystyle d\hat{\epsilon } =dx$. By making use of the convolution theorem, the above equation can be rewritten as follows,
		\begin{equation}
		\label{eq:fourier2}
		\begin{aligned}
		\mathcal{F}( \nu ) & =\int p\left(\mathbf{\hat{n}}\right)\int p_{\hat{\epsilon }}\left(\hat{\epsilon }\right) e^{-2\pi i\left(\hat{\epsilon } +g\left( y,\hat{\mathbf{n}}\right)\right) \cdot \nu } d\hat{\epsilon } d\hat{\mathbf{n}}\\
		& =\int p\left(\hat{\mathbf{n}}\right) e^{-2\pi ig\left( y,\hat{n}\right) \cdot \nu } d\mathbf{\hat{n}}\int p_{\hat{\epsilon }}\left(\hat{\epsilon }\right) e^{-2\pi i\hat{\epsilon } \cdot \nu } d\hat{\epsilon }.
		\end{aligned}
		\end{equation}
		Combing Eq. \ref{eq:fourier1} and \ref{eq:fourier2}, we have
		\begin{equation}
		\int p_{\hat{\epsilon }}\left(\hat{\epsilon }\right) e^{-2\pi i\hat{\epsilon } \cdot \nu } d\hat{\epsilon } =\frac{\int p( x|y) e^{-2\pi ix\cdot \nu } dx}{\int p\left(\hat{\mathbf{n}}\right) e^{-2\pi ig\left( y,\hat{\mathbf{n}}\right) \cdot \nu } d\hat{\mathbf{n}}}.
		\end{equation}
		Then, applying the inverse Fourier transform, we conclude
		\begin{equation}
		p_{\hat{\epsilon }}\left(\hat{\epsilon }\right) =\int e^{2\pi i\hat{\epsilon } \cdot \nu }\frac{\int p( x|y) e^{-2\pi ix\cdot \nu } dx}{\int p\left(\hat{\mathbf{n}}\right) e^{-2\pi ig\left( y,\hat{\mathbf{n}}\right) \cdot \nu } d\hat{\mathbf{n}}} d\nu .
		\end{equation}
		Based on  Bayes' theorem, $\displaystyle p( x|y) =\frac{p( x) p( y|x)}{p( y)} =\frac{p( x)\int p(\mathbf{n}) p_{\epsilon }( y-f( x,\mathbf{n})) d\mathbf{n}}{p( y)}$, and we further have 
		\begin{equation*}
		p_{\hat{\epsilon }}\left(\hat{\epsilon }\right)\! =\!\! \int \! e^{2\pi i\hat{\epsilon } \cdot \nu }\frac{\int \!\! \int p( x) p(\mathbf{n}) p_{\epsilon }( y-f( x,\mathbf{n})) e^{-2\pi ix\cdot \nu } d\mathbf{n} dx}{p( y)\int p\left(\hat{\mathbf{n}}\right) e^{-2\pi ig\left( y,\hat{\mathbf{n}}\right) \cdot \nu } d\hat{\mathbf{n}}} d\nu .
		\end{equation*}
	\end{proof-sketch}
	
	\section{Proof of Corollary 1}\label{sup:cor1}
	
	\begin{corollary}
		\label{sup:corollary:linear}
		Assume that CANM is linear Gaussian, i.e., 
		\begin{equation*}
		Y=aX+bN+\epsilon ,
		\end{equation*}
		where $X,N,\epsilon \sim \mathcal{N}(0,1)$, then their exist a backward CANM
		\begin{equation*}
		X=\frac{a}{a^{2} +b^{2} +1}Y+\frac{a}{\sqrt{a^{2} +b^{2} +1}}\hat{N}+\hat{\epsilon},
		\end{equation*} 
		where $\hat{N}, \hat{\epsilon} \sim \mathcal{N}(0,1)$ and $\hat{\epsilon }$ is independent of $Y$ and $\hat{N}$.
		
	\end{corollary}
	
	\begin{proof}
		Based on Theorem 1, the noise distribution on the reverse direction can be expressed as
		\begin{equation}
		p_{\hat{\epsilon }}\left(\hat{\epsilon }\right) =\int e^{2\pi i\hat{\epsilon } \cdot \nu }\frac{\int p( x|y) e^{-2\pi ix\cdot \nu } dx}{\int p\left(\hat{{n}}\right) e^{-2\pi ig\left( y,\hat{{n}}\right) \cdot \nu } d\hat{{n}}} d\nu .
		\end{equation}
		Based on the Bayes' theorem, $\displaystyle p( x|y) =\frac{p( x) p( y|x)}{p( y)} =\frac{p( x) p_{\tilde{\epsilon} }( y-ax)} {p( y)}$, where $\tilde{\epsilon} \sim \mathcal{N}(0,b^2+1)$ is the distribution of the $\tilde{\epsilon}=bn+\epsilon$. Without loss of generosity, let $g(y,\hat{n})=cy+d\hat{n}$, we have
		\begin{equation*}
		p_{\hat{\epsilon }}\left(\hat{\epsilon }\right)  =\int e^{2\pi i\hat{\epsilon } \cdot \nu }\frac{\int p_{\tilde{\epsilon}}( y-ax) p( x) e^{-2\pi ix\cdot \nu } dx}{p( y)\int p\left(\hat{n}\right) e^{-2\pi i\left( cy+d\hat{n}\right) \cdot \nu } d\hat{n}} d\nu.
		\end{equation*}
		The following derivation using the fact that the Fourier transform of the Gaussian distribution is
		\begin{equation*}
		\mathcal{F}_{x}\left[{\displaystyle \frac{1}{\sqrt{2\pi \sigma ^{2}}} e^{-\frac{1}{2\sigma ^{2}}( x-\mu )^{2}}}\right] (\nu )=e^{-2\pi i\mu \cdot \nu } e^{-2\pi ^{2} \sigma ^{2} \cdot \nu ^{2}},
		\end{equation*}
		then we have
		\begin{equation*}
		\begin{aligned}
		p_{\hat{\epsilon }}\left(\hat{\epsilon }\right) & =\int e^{2\pi i\hat{\epsilon } \cdot \nu }\frac{\int p_{\tilde{\epsilon}}(y-ax)p(x)e^{-2\pi ix\cdot \nu } dx}{p(y)\int p\left(\hat{n}\right) e^{-2\pi i\left( cy+d\hat{n}\right) \cdot \nu } d\hat{n}} d\nu \\
		& =\int e^{2\pi i\hat{\epsilon } \cdot \nu }\frac{\int {\displaystyle \frac{1}{\sqrt{2\pi \left( b^{2} +1\right)}} e^{-\frac{(y-ax)^{2}}{2\left( b^{2} +1\right)}}\frac{1}{\sqrt{2\pi }} e^{-\frac{x^{2}}{2}}} e^{-2\pi ix\cdot \nu } dx}{{\displaystyle \frac{1}{\sqrt{2\pi \left( a^{2} +b^{2} +1\right)}} e^{-\frac{(y-ax)^{2}}{2\left( a^{2} +b^{2} +1\right)}}} e^{-2\pi icy\cdot \nu }\int p\left(\hat{n}\right) e^{-2\pi id\hat{n} \cdot \nu } d\hat{n}} d\nu \\
		& =\int e^{2\pi i\hat{\epsilon } \cdot \nu }\frac{\int {\displaystyle \frac{\sqrt{2\pi \left( a^{2} +b^{2} +1\right)}}{\sqrt{2\pi }\sqrt{2\pi \left( b^{2} +1\right)}} e^{-\frac{(y-ax)^{2}}{2\left( b^{2} +1\right)} -\frac{x^{2}}{2} +\frac{y^{2}}{2\left( a^{2} +b^{2} +1\right)}}} e^{-2\pi ix\cdot \nu } dx}{e^{-2\pi icy\cdot \nu }\int p\left(\hat{n}\right) e^{-2\pi id\hat{n} \cdot \nu } d\hat{n}} d\nu \\
		& =\int e^{2\pi i\hat{\epsilon } \cdot \nu }\frac{\int {\displaystyle \frac{\sqrt{\left( a^{2} +b^{2} +1\right)}}{\sqrt{2\pi \left( b^{2} +1\right)}} e^{-\left( a^{2} +b^{2} +1\right) x^{2} -a^{2} y^{2} +2axy\left( a^{2} +b^{2} +1\right)}} e^{-2\pi ix\cdot \nu } dx}{e^{-2\pi icy\cdot \nu }\int p\left(\hat{n}\right) e^{-2\pi id\hat{n} \cdot \nu } d\hat{n}} d\nu \\
		& =\int e^{2\pi i\hat{\epsilon } \cdot \nu }\frac{\int {\displaystyle \frac{\sqrt{\left( a^{2} +b^{2} +1\right)}}{\sqrt{2\pi \left( b^{2} +1\right)}} e^{-\frac{\left(\left( a^{2} +b^{2} +1\right) x-ay\right)^{2}}{2\left( b^{2} +1\right)\left( a^{2} +b^{2} +1\right)}}} e^{-2\pi ix\cdot \nu } dx}{e^{-2\pi icy\cdot \nu }\int p\left(\hat{n}\right) e^{-2\pi id\hat{n} \cdot \nu } d\hat{n}} d\nu \\
		& =\int e^{2\pi i\hat{\epsilon } \cdot \nu }\frac{\int {\displaystyle \frac{1}{\sqrt{2\pi \left( b^{2} +1\right) /\left( a^{2} +b^{2} +1\right)}} e^{-\frac{\left( x-\frac{a}{a^{2} +b^{2} +1} y\right)^{2}}{2\left( b^{2} +1\right) /\left( a^{2} +b^{2} +1\right)}}} e^{-2\pi i\left( x-\frac{a}{a^{2} +b^{2} +1} y\right) \cdot \nu } e^{-2\pi i\frac{a}{a^{2} +b^{2} +1} y\cdot \nu } dx}{e^{-2\pi icy\cdot \nu }\int p\left(\hat{n}\right) e^{-2\pi id\hat{n} \cdot \nu } d\hat{n}} d\nu \\
		& =\int e^{2\pi i\hat{\epsilon } \cdot \nu }\frac{e^{-2\pi i\frac{a}{a^{2} +b^{2} +1} y\cdot \nu }{\displaystyle e^{-2\pi ^{2}\left( b^{2} +1\right) /\left( a^{2} +b^{2} +1\right) \cdot \nu ^{2}}}}{e^{-2\pi icy\cdot \nu }\int p\left(\hat{n}\right) e^{-2\pi id\hat{n} \cdot \nu } d\hat{n}} d\nu \\
		& =\int e^{2\pi i\hat{\epsilon } \cdot \nu }\frac{{\displaystyle e^{-2\pi i\frac{a}{a^{2} +b^{2} +1} y\cdot \nu } e^{-2\pi ^{2}\left( b^{2} +1\right) /\left( a^{2} +b^{2} +1\right) \cdot \nu ^{2}}}}{e^{-2\pi icy\cdot \nu }\int {\displaystyle \frac{1}{\sqrt{2\pi d^{2}}} e^{-\frac{\left( d\hat{n}\right)^{2}}{2d^{2}}}} e^{-2\pi id\hat{n} \cdot \nu } d\left( d\hat{n}\right)} d\nu \\
		& =\int e^{2\pi i\hat{\epsilon } \cdot \nu }\frac{{\displaystyle e^{-2\pi i\frac{a}{a^{2} +b^{2} +1} y\cdot \nu } e^{-2\pi ^{2}\left( b^{2} +1\right) /\left( a^{2} +b^{2} +1\right) \cdot \nu ^{2}}}}{e^{-2\pi icy\cdot \nu } e^{-2\pi ^{2} d^{2} \cdot \nu ^{2}}} d\nu .
		\end{aligned}
		\end{equation*}
		
		Let $\displaystyle c=\frac{a}{a^{2} +b^{2} +1} ,d^{2} =\frac{a^{2}}{a^{2} +b^{2} +1}$, we obtain
		\begin{equation*}
		\begin{aligned}
		& \int e^{2\pi i\hat{\epsilon } \cdot \nu }\frac{{\displaystyle e^{-2\pi i\frac{a}{a^{2} +b^{2} +1} y\cdot \nu } e^{-2\pi ^{2}\left( b^{2} +1\right) /\left( a^{2} +b^{2} +1\right) \cdot \nu ^{2}}}}{e^{-2\pi icy\cdot \nu } e^{-2\pi ^{2} d^{2} \cdot \nu ^{2}}} d\nu \\
		= & \int e^{2\pi i\hat{\epsilon } \cdot \nu -2\pi ^{2} \cdot \nu ^{2}} d\nu \\
		= & \int e^{-\left(\sqrt{2} \pi \nu -\frac{i\hat{\epsilon }}{\sqrt{2}}\right)^{2} -\frac{\hat{\epsilon }^{2}}{2}} d\nu \\
		= & \frac{1}{\sqrt{2\pi }} e^{-\frac{\hat{\epsilon }^{2}}{2}}.
		\end{aligned}
		\end{equation*}
		Thus, we have $p_{\hat{\epsilon }}\left(\hat{\epsilon }\right)=\frac{1}{\sqrt{2\pi }} e^{-\frac{\hat{\epsilon }^{2}}{2}}$, which is a Gaussian distribution and independent of $y,\hat{\mathbf{n}}$.
	\end{proof}
	\section{Proof of Corollary 2}\label{sup:cor2}
	\begin{corollary}
		\label{sup:corollary:anm}
		Suppose that there is no unmeasured intermediate noises in CANM, if the solution of Eq. \ref{eq:thm1} exists, then the triple $\displaystyle ( f,p_{X} ,p_{\epsilon})$ must satisfy the differential equation from ANM  \cite[Theorem~1]{hoyer2009nonlinear} for all $\displaystyle x,y$ with $\displaystyle \nu ''( y-f( x)) f'( x) \neq 0$:
		\begin{equation}
		\xi ^{\prime \prime \prime } =\xi ^{\prime \prime }\left( -\frac{\nu ^{\prime \prime \prime } f^{\prime }}{\nu ^{\prime \prime }} +\frac{f^{\prime \prime }}{f^{\prime }}\right) -2\nu ^{\prime \prime } f^{\prime \prime } f^{\prime } +\nu ^{\prime } f^{\prime \prime \prime } +\frac{\nu ^{\prime } \nu ^{\prime \prime \prime } f^{\prime \prime } f^{\prime }}{\nu ^{\prime \prime }} -\frac{\nu ^{\prime }\left( f^{\prime \prime }\right)^{2}}{f^{\prime }} ,
		\end{equation}
		where $\displaystyle \nu \coloneq \log p_{\epsilon} ,\ \xi \coloneq \log p_{X}$ 
	\end{corollary}
	
	\begin{proof}
		Since no unobserved intermediate noises, based on Theorem 1, we have
		\begin{equation}
		\label{eq:noise of anm}
		p_{\hat{\epsilon }}\left(\hat{\epsilon }\right) =e^{2\pi i\left(\hat{\epsilon } -g( y)\right) \cdot \nu }\int \frac{p( x) p_{\epsilon }( y-f( x))}{p( y)} e^{-2\pi ix\cdot \nu } dx.
		\end{equation}
		Let $\displaystyle \hat{\epsilon } =x-g( y)$, then based on the Fourier inverse theorem, the existence of the solution of Eq. \ref{eq:noise of anm} is equivalent to the existence of following equation,
		\begin{equation}
		p_{\hat{\epsilon }}( x-g( y)) =\frac{p( x) p_{\epsilon }( y-f( x))}{p( y)} ,
		\end{equation}
		which is the standard identifiability for additive noise model, then applying the \cite[Theorem~1]{hoyer2009nonlinear}, the triple $\displaystyle ( f,p_{x} ,p_{\epsilon})$ must satisfy the following differential equation for all $\displaystyle x,y$ with $\displaystyle \nu ''( y-f( x)) f'( x) \neq 0$:
		\begin{equation*}
		\xi ^{\prime \prime \prime } =\xi ^{\prime \prime }\left( -\frac{\nu ^{\prime \prime \prime } f^{\prime }}{\nu ^{\prime \prime }} +\frac{f^{\prime \prime }}{f^{\prime }}\right) -2\nu ^{\prime \prime } f^{\prime \prime } f^{\prime } +\nu ^{\prime } f^{\prime \prime \prime } +\frac{\nu ^{\prime } \nu ^{\prime \prime \prime } f^{\prime \prime } f^{\prime }}{\nu ^{\prime \prime }} -\frac{\nu ^{\prime }\left( f^{\prime \prime }\right)^{2}}{f^{\prime }} .
		\end{equation*}
	\end{proof}
	
\end{document}